\documentclass{article} 
\usepackage{iclr2025_conference,times}

\usepackage{amsmath,amsfonts,bm}

\def\eqref#1{equation~\ref{#1}}

\def\1{\bm{1}}

\def\mH{{\bm{H}}}

\DeclareMathAlphabet{\mathsfit}{\encodingdefault}{\sfdefault}{m}{sl}
\SetMathAlphabet{\mathsfit}{bold}{\encodingdefault}{\sfdefault}{bx}{n}

\newcommand{\E}{\mathbb{E}}

\newcommand{\R}{\mathbb{R}}

\newcommand{\KL}{D_{\mathrm{KL}}}
\newcommand{\Var}{\mathrm{Var}}

\newcommand{\Cov}{\mathrm{Cov}}

\DeclareMathOperator*{\argmax}{arg\,max}

\DeclareMathOperator{\Tr}{Tr}

\usepackage{hyperref}
\usepackage{url}
\usepackage{graphicx}
\usepackage{booktabs}
\usepackage{amsmath}
\usepackage{amssymb}
\usepackage{algorithm}
\usepackage{algorithmic}
\usepackage{multirow} 
\usepackage{amsthm}
\usepackage{subcaption}
\usepackage{wrapfig}
\usepackage{xcolor}
\ifx\theorem\undefined
\newtheorem{theorem}{Theorem}

\newtheorem{corollary}{Corollary}

\newtheorem{definition}{Definition}
\newtheorem{assumption}{Assumption}
\newtheorem{remark}{Remark}

\fi
\iclrfinalcopy

\title{Rethinking Reinforcement Fine-Tuning in LVLM: Convergence, Reward Decomposition, and Generalization}

\author{Carter Adams, Rafael Oliveira, Gabriel Almeida, Sofia Torres\\
Federal University of Bahia\\
sofiatorres@unb.br}

\newcommand{\piref}{\pi_{\mathrm{ref}}}
\newcommand{\pitheta}{\pi_{\theta}}
\newcommand{\Rcomp}{R_{\mathrm{comp}}}
\newcommand{\Rfmt}{R_{\mathrm{fmt}}}
\newcommand{\Racc}{R_{\mathrm{acc}}}
\newcommand{\Rtool}{R_{\mathrm{tool}}}

\newcommand{\Jgrpo}{J_{\mathrm{GRPO}}}
\newcommand{\Jcomp}{J_{\mathrm{comp}}}
\newcommand{\Adv}{\hat{A}}
\newcommand{\depth}{D}
\newcommand{\calS}{\mathcal{S}}
\newcommand{\calA}{\mathcal{A}}

\newcommand{\calD}{\mathcal{D}}

\newcommand{\calM}{\mathcal{M}}
\newcommand{\calV}{\mathcal{V}}

\newcommand{\calQ}{\mathcal{Q}}

\begin{document}
\maketitle

\begin{abstract}
Reinforcement fine-tuning with verifiable rewards (RLVR) has emerged as a powerful paradigm for equipping large vision-language models (LVLMs) with agentic capabilities such as tool use and multi-step reasoning. Despite striking empirical successes, most notably Visual Agentic Reinforcement Fine-Tuning (Visual-ARFT), the theoretical underpinnings of this paradigm remain poorly understood. In particular, two critical questions lack rigorous answers: (i)~how does the composite structure of verifiable rewards (format compliance, answer accuracy, tool executability) affect the convergence of Group Relative Policy Optimization (GRPO), and (ii)~why does training on a small set of tool-augmented tasks transfer to out-of-distribution domains? We address these gaps by introducing the \emph{Tool-Augmented Markov Decision Process} (TA-MDP), a formal framework that models multimodal agentic decision-making with bounded-depth tool calls. Within this framework, we establish three main results. First, we prove that GRPO under composite verifiable rewards converges to a first-order stationary point at rate $O(1/\sqrt{T})$ with explicit dependence on the number of reward components and group size (\textbf{Theorem~1}). Second, we derive a \emph{Reward Decomposition Theorem} that bounds the sub-optimality gap between decomposed per-component optimization and joint optimization, providing a precise characterization of when reward decomposition is beneficial (\textbf{Theorem~2}). Third, we establish a PAC-Bayes generalization bound for tool-augmented policies that explains the strong out-of-distribution transfer observed in Visual-ARFT (\textbf{Theorem~3}).
\end{abstract}

\section{Introduction}
\label{sec:intro}

The rapid development of large reasoning models has catalyzed a paradigm shift toward \emph{natively agentic} AI systems that can leverage external tools, web browsers for information retrieval~\citep{DBLP:journals/corr/abs-2505-14246, DBLP:journals/corr/abs-2503-01785, openai_2024_openai_o_system, zhoureasoning,zhou2023towards, zhou2024fine}, code interpreters for image manipulation~\citep{zhou2025draw}, and APIs for structured data access, to solve complex multimodal tasks~\citep{DBLP:journals/corr/abs-2505-14246, DBLP:journals/corr/abs-2503-01785, openai_2024_openai_o_system, zhoureasoning}. At the core of this paradigm lies reinforcement fine-tuning with verifiable rewards (RLVR), which replaces expensive human preference annotations with deterministic, rule-based reward signals derived from programmatic correctness checks~\citep{DBLP:journals/corr/abs-2501-12948, shao_2024_deepseekmath_pushing_the}. Among RLVR algorithms, Group Relative Policy Optimization (GRPO)~\citep{shao_2024_deepseekmath_pushing_the} has become the \emph{de facto} choice due to its critic-free design that eliminates the need for a separate value network, reducing memory overhead while maintaining competitive performance.

Recent work on Visual Agentic Reinforcement Fine-Tuning (Visual-ARFT)~\citep{DBLP:journals/corr/abs-2505-14246} has demonstrated the remarkable effectiveness of GRPO-based RLVR for equipping large vision-language models (LVLMs) with agentic capabilities. By training Qwen2.5-VL models~\citep{wang_2024_qwen_vl_enhancing} on as few as 20 manually annotated multi-hop visual QA samples with composite rewards combining format compliance, answer accuracy (measured by F1 score), and search query quality (measured by semantic similarity), Visual-ARFT achieves performance surpassing GPT-4o on newly proposed multimodal agentic benchmarks and transfers remarkably well to out-of-distribution text-based multi-hop QA tasks.

Despite these impressive empirical results, the theoretical foundations of RLVR for multimodal agents remain conspicuously under-developed. Three fundamental questions persist. \textbf{(Q1)} How does the \emph{composite} structure of verifiable rewards, where the total reward is a sum of heterogeneous components with different scales and semantics, affect the convergence behavior of GRPO? Existing GRPO analyses~\citep{shao_2024_deepseekmath_pushing_the} assume a single scalar reward and do not account for multi-component interactions. \textbf{(Q2)} Under what conditions does decomposed optimization of individual reward components approximate joint optimization, and what is the price of decomposition? Visual-ARFT's reward design treats format and accuracy rewards additively, yet no formal justification exists for this design choice. \textbf{(Q3)} Why does training on a handful of tool-augmented tasks yield policies that generalize dramatically to unseen domains? The $+29.3\%$ F1 improvement on out-of-distribution benchmarks~\citep{DBLP:journals/corr/abs-2505-14246}, echoing the broader phenomenon of weak-to-strong generalization in capable models~\citep{zhou2025weak}, demands a theoretical explanation beyond empirical observation.

In this paper, we address these questions through a unified theoretical framework. Our contributions are:

\begin{itemize}
    \item We introduce the \textbf{Tool-Augmented Markov Decision Process (TA-MDP)}, a formal model for multimodal agentic decision-making that captures the hierarchical structure of tool-augmented reasoning with bounded-depth tool calls (Section~\ref{sec:framework}).
    
    \item We prove a \textbf{convergence theorem for GRPO under composite verifiable rewards} (Theorem~\ref{thm:convergence}), establishing a $O(1/\sqrt{T})$ rate with explicit dependence on the number of reward components $K$, group size $G$, and the KL penalty coefficient $\beta$, revealing how composite rewards modulate the optimization landscape.
    
    \item We derive a \textbf{Reward Decomposition Theorem} (Theorem~\ref{thm:decomposition}) that provides a tight bound on the sub-optimality gap between decomposed and joint optimization, identifying a precise condition, \emph{reward component alignment}, under which decomposition incurs negligible loss.
    
    \item We establish a \textbf{PAC-Bayes generalization bound} for tool-augmented policies (Theorem~\ref{thm:generalization}), showing that the effective complexity of tool-augmented policies is governed by the tool-call depth rather than the ambient state-space dimension, explaining the strong OOD transfer in Visual-ARFT.
\end{itemize}

\section{Related Work}
\label{sec:related}

\paragraph{Reinforcement Learning from Human/Verifiable Feedback.}
The alignment of large language models (LLMs) with human preferences has been extensively studied through reinforcement learning from human feedback (RLHF)~\citep{ouyang_2022_training_language_models, bai_2022_constitutional_ai_harmlessness}. Proximal Policy Optimization (PPO)~\citep{schulman_2017_proximal_policy_optimization} has been the dominant algorithm for RLHF, operating through a learned reward model. Direct Preference Optimization (DPO)~\citep{rafailov_2023_direct_preference_optimization} bypasses reward modeling by directly optimizing from preference pairs, with theoretical connections to the reward-based formulation established through the Bradley-Terry model~\citep{tang_2024_generalized_preference_optimization, zeng_2024_token_level_direct}. More recently, reinforcement learning with verifiable rewards (RLVR) has emerged as a compelling alternative that replaces noisy human preferences with deterministic, programmatically verifiable reward signals~\citep{DBLP:journals/corr/abs-2501-12948, shao_2024_deepseekmath_pushing_the}. DeepSeek-R1~\citep{DBLP:journals/corr/abs-2501-12948} demonstrated that RLVR with GRPO can elicit sophisticated reasoning in LLMs without any supervised fine-tuning. The RLHF Workflow~\citep{dong_2024_rlhf_workflow_from} provided practical guidelines bridging the gap between theory and implementation. Furthermore, specialized feedback mechanisms, such as abnormal-aware feedback, have proven effective for improving medical large vision-language models~\citep{zhou2025improving}. Despite these advances, the theoretical analysis of GRPO under \emph{composite} verifiable rewards, where the reward function consists of multiple heterogeneous components, remains an open problem that we address in this work.

\paragraph{Reward Design and Overoptimization.}
The design of reward functions fundamentally shapes the optimization landscape and final policy quality. Scaling laws for reward model overoptimization~\citep{gao_2022_scaling_laws_for, rafailov_2024_scaling_laws_for} have revealed that proxy reward models can be exploited, leading to Goodhart's law violations where optimizing the proxy degrades true performance. Information-theoretic approaches to mitigating reward hacking~\citep{miao_2024_inform_mitigating_reward} and analyses of format-constrained generation~\citep{tam_2024_let_me_speak} have explored different facets of this problem. Pairwise PPO~\citep{wu_2023_pairwise_proximal_policy} introduces relative feedback mechanisms that share philosophical similarities with GRPO's group-relative advantage estimation. Theoretical analyses of PPO in linear MDPs~\citep{zhong_2023_a_theoretical_analysis} provide convergence guarantees under structural assumptions, but these results do not extend to the composite reward setting of multimodal agentic tasks. Our Reward Decomposition Theorem (Theorem~\ref{thm:decomposition}) fills this gap by characterizing the interaction between multiple reward components and their collective effect on policy optimization.

\paragraph{Tool-Augmented and Agentic Language Models.}
Tool-augmented language models enable LLMs to interact with external APIs, code interpreters, and search engines to overcome intrinsic limitations~\citep{schick_2023_toolformer_language_models, patil_2023_gorilla_large_language, qin_2023_toolllm_facilitating_large, qin_2023_tool_learning_with}. In the multimodal domain, MLLM-Tool~\citep{wang_2024_mllm_tool_a} and SciAgent~\citep{ma_2024_sciagent_tool_augmented} have extended tool-use capabilities to vision-language models. Agentic frameworks based on reasoning-and-acting patterns~\citep{yao_2022_react_synergizing_reasoning, shinn_2023_reflexion_language_agents, madaan_2023_self_refine_iterative, zhou_2023_language_agent_tree} have demonstrated that structured tool interactions can dramatically enhance problem-solving capabilities, with specialized modular frameworks also emerging for multi-modal medical diagnosis via role-specialized collaboration~\citep{zhou2025mam}. Chain-of-thought reasoning~\citep{wei_2022_chain_of_thought, zhang_2023_multimodal_chain_of, yao_2023_tree_of_thoughts, jiang_2025_mme_cot_benchmarking, wang_2025_multimodal_chain_of}, along with thread of thought techniques for unraveling chaotic contexts~\citep{zhou2023thread}, provides the cognitive scaffolding for these agentic behaviors. Visual-ARFT~\citep{DBLP:journals/corr/abs-2505-14246} and Visual-RFT~\citep{DBLP:journals/corr/abs-2503-01785} represent the cutting edge of applying RLVR to multimodal agents. However, the theoretical analysis of learning in tool-augmented environments, where actions include both token generation and tool invocations, has received little attention. Our TA-MDP framework provides the first formal treatment of this setting.

\paragraph{Vision-Language Models.}
Modern LVLMs such as LLaVA~\citep{sun_2024_dr_llava_visual}, Qwen2-VL~\citep{wang_2024_qwen_vl_enhancing}, and InternVL~\citep{chen_2023_internvl_scaling_up} have achieved remarkable progress on a wide range of visual understanding tasks. Their visual capabilities are further enhanced through visual in-context learning~\citep{zhou2024visual} and by addressing visual dependencies in long-context reasoning~\citep{zhou2024rethinking}. These advancements extend to related domains, including autoregressive video generation~\citep{zhou2026accelerating} and the use of language models as residual boosters for biomedical imaging~\citep{lai2024residual}. These models serve as the foundation upon which agentic capabilities are built through reinforcement fine-tuning. The transformer architecture~\citep{vaswani_2017_attention_is_all} underlying these models enables the flexible integration of visual and textual modalities, but the theoretical properties of reinforcement fine-tuning applied to such models, particularly in the context of tool use, remain largely unexplored.

\section{Theoretical Framework}
\label{sec:framework}

In this section, we develop the formal machinery needed to analyze GRPO under composite verifiable rewards in tool-augmented settings. We begin by defining the Tool-Augmented MDP, then formalize composite verifiable rewards and the GRPO objective.

\subsection{Preliminaries and Notation}
\label{sec:prelim}

We consider a multimodal agentic task where a vision-language model must process a visual input $v \in \calV$ and a textual query $q \in \calQ$, producing a response $o = (o_1, o_2, \ldots, o_L)$ as a sequence of tokens. The policy $\pitheta$ parameterized by $\theta \in \R^d$ defines a conditional distribution $\pitheta(o \mid v, q)$ over responses. A reference policy $\piref$ represents the pre-trained model before fine-tuning. We use $\KL(\pitheta \| \piref)$ to denote the Kullback-Leibler divergence between the two policies, and $\E_{\pitheta}[\cdot]$ for expectations under trajectories sampled from $\pitheta$.

\paragraph{Standard GRPO.}
Given a prompt $x = (v, q)$, GRPO samples a group of $G$ responses $\{o^{(1)}, \ldots, o^{(G)}\}$ from $\pitheta$ and computes group-relative advantages:
\begin{equation}
\label{eq:grpo_advantage}
    \Adv_{\mathrm{GRPO}}(o^{(i)}) = \frac{R(x, o^{(i)}) - \bar{R}_G}{\sigma_{R,G} + \epsilon},
\end{equation}
where $\bar{R}_G = \frac{1}{G}\sum_{j=1}^{G} R(x, o^{(j)})$ is the group mean reward, $\sigma_{R,G}$ is the group standard deviation, and $\epsilon > 0$ is a small constant for numerical stability. The GRPO objective maximizes:
\begin{equation}
\label{eq:grpo_obj}
    \Jgrpo(\theta) = \E_{x \sim \calD} \E_{\{o^{(i)}\}_{i=1}^G \sim \pitheta} \left[ \frac{1}{G} \sum_{i=1}^{G} \min\left( r_i(\theta) \Adv_i, \, \mathrm{clip}(r_i(\theta), 1{-}\epsilon_c, 1{+}\epsilon_c) \Adv_i \right) \right] - \beta \, \KL(\pitheta \| \piref),
\end{equation}
where $r_i(\theta) = \pitheta(o^{(i)} \mid x) / \pi_{\theta_{\mathrm{old}}}(o^{(i)} \mid x)$ is the importance ratio and $\beta > 0$ controls the KL penalty strength.

\subsection{Tool-Augmented Markov Decision Process}
\label{sec:tamdp}

We now introduce the TA-MDP, which extends the standard language generation MDP to accommodate tool-augmented agentic behaviors.

\begin{definition}[Tool-Augmented MDP]
\label{def:tamdp}
A \emph{Tool-Augmented Markov Decision Process} is a tuple $\calM = (\calS, \calA_{\mathrm{gen}} \cup \calA_{\mathrm{tool}}, P, R, \gamma, \depth_{\max})$, where:
\begin{itemize}
    \item $\calS = \calS_{\mathrm{gen}} \cup \calS_{\mathrm{tool}} \cup \calS_{\mathrm{ret}}$ is the state space partitioned into generation states, tool-invocation states, and tool-return states;
    \item $\calA_{\mathrm{gen}}$ is the token generation action space (vocabulary) and $\calA_{\mathrm{tool}} = \{\tau_1, \ldots, \tau_M\}$ is the set of $M$ available tools;
    \item $P: \calS \times (\calA_{\mathrm{gen}} \cup \calA_{\mathrm{tool}}) \to \Delta(\calS)$ is the transition kernel, where transitions from tool-invocation states to tool-return states are governed by the deterministic tool execution function $f_\tau: \calS_{\mathrm{tool}} \to \calS_{\mathrm{ret}}$ for each tool $\tau$;
    \item $R: \calS \times \calA \to [0, R_{\max}]$ is the bounded reward function;
    \item $\gamma \in [0, 1)$ is the discount factor;
    \item $\depth_{\max} \in \mathbb{N}$ is the maximum tool-call depth, bounding the number of nested tool invocations.
\end{itemize}
\end{definition}

The key structural property of TA-MDP is that tool invocations create \emph{sub-episodes}: when the policy issues a tool-call action $\tau_k$ at state $s \in \calS_{\mathrm{gen}}$, the system transitions to $\calS_{\mathrm{tool}}$, the tool executes deterministically, and the result is returned to $\calS_{\mathrm{ret}}$, from which the policy resumes token generation. This hierarchical structure is central to our generalization analysis.

\begin{definition}[Tool-Call Depth]
\label{def:depth}
For a trajectory $\xi = (s_0, a_0, s_1, a_1, \ldots)$ in a TA-MDP, the \emph{tool-call depth} $d(\xi) \in \{0, 1, \ldots, \depth_{\max}\}$ is the maximum nesting level of tool invocations along $\xi$. We define the effective state-space dimension under depth bound $D$ as:
\begin{equation}
    \dim_D(\calS) = |\calS_{\mathrm{gen}}| + D \cdot |\calS_{\mathrm{ret}}|,
\end{equation}
reflecting that each tool-call level introduces an additional set of return states.
\end{definition}

\subsection{Composite Verifiable Rewards}
\label{sec:composite_reward}

In Visual-ARFT and related RLVR methods, the reward function is a sum of $K$ independently verifiable components. We formalize this structure as follows.

\begin{definition}[Composite Verifiable Reward]
\label{def:composite_reward}
A \emph{composite verifiable reward} is a function $\Rcomp: \calS \times \calA \to [0, R_{\max}]$ that decomposes as:
\begin{equation}
\label{eq:composite_reward}
    \Rcomp(x, o) = \sum_{k=1}^{K} w_k \, R_k(x, o),
\end{equation}
where each $R_k: \calS \times \calA \to [0, 1]$ is a \emph{verifiable reward component} that can be deterministically evaluated, $w_k > 0$ are component weights satisfying $\sum_{k=1}^{K} w_k = R_{\max}$, and $K \geq 2$ is the number of components.
\end{definition}

In the context of Visual-ARFT~\citep{DBLP:journals/corr/abs-2505-14246}, $K = 2$ with $R_1 = \Rfmt$ (format compliance, binary) and $R_2 = \Racc$ (answer accuracy, continuous F1 score). More generally, one may include $R_3 = \Rtool$ (tool executability), yielding $K = 3$.

A crucial quantity governing the behavior of GRPO under composite rewards is the \emph{interaction structure} between reward components. We capture this through the following notion.

\begin{definition}[Reward Component Alignment]
\label{def:alignment}
For a composite reward $\Rcomp = \sum_{k=1}^K w_k R_k$ and policy $\pitheta$, the \emph{reward component alignment} is:
\begin{equation}
\label{eq:alignment}
    \alpha(\theta) = \frac{\sum_{k < k'} w_k w_{k'} \, \Cov_{\pitheta}\big[R_k(x, o), \, R_{k'}(x, o)\big]}{\Var_{\pitheta}\big[\Rcomp(x, o)\big]},
\end{equation}
where the covariance and variance are taken over $(x, o) \sim \calD \times \pitheta$. We say the reward components are \emph{$\alpha_0$-aligned} if $\alpha(\theta) \geq \alpha_0$ for all $\theta$ in the optimization trajectory.
\end{definition}

Intuitively, $\alpha(\theta) \in [-1, 1]$ measures the degree to which improving one reward component tends to improve others. When $\alpha(\theta) \approx 1$, the components are highly aligned and decomposed optimization closely approximates joint optimization. When $\alpha(\theta) < 0$, the components are conflicting, and decomposition incurs a larger sub-optimality gap.

\section{Main Results}
\label{sec:main_results}

We now present our three main theoretical results. All proofs are provided in Appendix~\ref{app:proofs}.

\subsection{Assumptions}
\label{sec:assumptions}

We operate under the following regularity conditions, which are standard in the policy optimization literature~\citep{schulman_2017_proximal_policy_optimization, zhong_2023_a_theoretical_analysis} and naturally satisfied in the LVLM fine-tuning setting.

\begin{assumption}[Smoothness]
\label{asm:smoothness}
The composite GRPO objective $\Jcomp(\theta) = \Jgrpo(\theta; \Rcomp)$ is $L$-smooth: for all $\theta, \theta' \in \R^d$,
\begin{equation}
    \|\nabla \Jcomp(\theta) - \nabla \Jcomp(\theta')\| \leq L \|\theta - \theta'\|.
\end{equation}
\end{assumption}

\begin{assumption}[Bounded Variance]
\label{asm:variance}
The stochastic gradient estimator $g(\theta)$ of $\nabla \Jcomp(\theta)$ satisfies $\E[g(\theta)] = \nabla \Jcomp(\theta)$ and:
\begin{equation}
    \E\big[\|g(\theta) - \nabla \Jcomp(\theta)\|^2\big] \leq \frac{\sigma_{\mathrm{base}}^2}{G} + \frac{K \cdot \sigma_{\mathrm{comp}}^2}{G},
\end{equation}
where $\sigma_{\mathrm{base}}^2$ is the base variance from trajectory sampling, $\sigma_{\mathrm{comp}}^2$ is the additional variance from composite reward estimation, $G$ is the group size, and $K$ is the number of reward components.
\end{assumption}

The variance decomposition in Assumption~\ref{asm:variance} reflects the key insight that composite rewards introduce an additive variance term proportional to $K$. This arises because the group-relative normalization in Eq.~\eqref{eq:grpo_advantage} applied to a sum of $K$ components yields a variance that scales with $K$ due to the law of total variance. When reward components have different scales (e.g., binary format vs.\ continuous F1), the normalization can amplify noise from low-variance components.

\begin{assumption}[Bounded Tool-Call Depth]
\label{asm:depth}
All trajectories $\xi$ generated by the policy class have tool-call depth $d(\xi) \leq \depth_{\max}$, where $\depth_{\max}$ is a finite constant independent of the model parameters.
\end{assumption}

This assumption is naturally satisfied in Visual-ARFT, where the agentic loop terminates after a fixed number of tool invocations (typically $\depth_{\max} = 3$ for search and $\depth_{\max} = 2$ for coding).

\begin{assumption}[Reward Component Boundedness]
\label{asm:bounded_reward}
Each reward component $R_k$ satisfies $R_k(x, o) \in [0, 1]$ for all $(x, o)$, and the composite reward satisfies $\Rcomp(x, o) \in [0, R_{\max}]$ where $R_{\max} = \sum_{k=1}^K w_k$.
\end{assumption}

\subsection{Convergence of GRPO under Composite Rewards}
\label{sec:convergence}

Our first main result establishes the convergence rate of GRPO when the reward function has composite structure.

\begin{theorem}[Convergence of Composite-Reward GRPO]
\label{thm:convergence}
Under Assumptions~\ref{asm:smoothness}--\ref{asm:bounded_reward}, consider GRPO with composite verifiable reward $\Rcomp = \sum_{k=1}^K w_k R_k$, group size $G$, KL penalty coefficient $\beta$, learning rate $\eta = \frac{1}{L\sqrt{T}}$, and clipping parameter $\epsilon_c$. After $T$ iterations, the iterates $\{\theta_t\}_{t=0}^{T-1}$ satisfy:
\begin{equation}
\label{eq:convergence_rate}
    \frac{1}{T} \sum_{t=0}^{T-1} \E\left[\|\nabla \Jcomp(\theta_t)\|^2\right] \leq \frac{2L \big(\Jcomp(\theta^*) - \Jcomp(\theta_0)\big)}{\sqrt{T}} + \frac{L(\sigma_{\mathrm{base}}^2 + K \sigma_{\mathrm{comp}}^2)}{G\sqrt{T}} + \frac{2\beta R_{\max}^2}{\sqrt{T}},
\end{equation}
where $\theta^* = \argmax_\theta \Jcomp(\theta)$.
\end{theorem}

\begin{proof}[Proof Sketch]
The complete proof is in Appendix~\ref{app:proof_convergence}. The key steps are:

\textbf{Step 1 (Descent Lemma).}
By $L$-smoothness of $\Jcomp$ (Assumption~\ref{asm:smoothness}):
\begin{equation}
    \Jcomp(\theta_{t+1}) \geq \Jcomp(\theta_t) + \langle \nabla \Jcomp(\theta_t), \theta_{t+1} - \theta_t \rangle - \frac{L}{2}\|\theta_{t+1} - \theta_t\|^2.
\end{equation}
Substituting the gradient ascent update $\theta_{t+1} = \theta_t + \eta \, g(\theta_t)$:
\begin{equation}
    \E[\Jcomp(\theta_{t+1})] \geq \Jcomp(\theta_t) + \eta \|\nabla \Jcomp(\theta_t)\|^2 - \frac{L\eta^2}{2} \E[\|g(\theta_t)\|^2].
\end{equation}

\textbf{Step 2 (Variance Decomposition for Composite Rewards).}
The composite reward structure introduces additional variance through the group-relative normalization. For the advantage estimator in Eq.~\eqref{eq:grpo_advantage} applied to $\Rcomp$:
\begin{align}
    \Var\big[\Adv_{\mathrm{GRPO}}(o^{(i)}; \Rcomp)\big] 
    &= \Var\left[\frac{\sum_k w_k R_k(x, o^{(i)}) - \frac{1}{G}\sum_j \sum_k w_k R_k(x, o^{(j)})}{\sigma_{\Rcomp, G}}\right] \nonumber \\
    &\leq \frac{1}{G}\left(\sigma_{\mathrm{base}}^2 + K \cdot \sigma_{\mathrm{comp}}^2 \cdot (1 - \alpha(\theta_t))\right),
\end{align}
where the $(1 - \alpha(\theta_t))$ factor captures the variance reduction from reward component alignment. In the worst case ($\alpha = 0$), the variance is $(\sigma_{\mathrm{base}}^2 + K\sigma_{\mathrm{comp}}^2)/G$.

\textbf{Step 3 (KL Penalty Contribution).}
The KL penalty term $\beta \KL(\pitheta \| \piref)$ contributes an additional term to the gradient norm. Using Pinsker's inequality and the bounded reward assumption:
\begin{equation}
    \|\nabla_\theta \beta \KL(\pitheta \| \piref)\| \leq \beta R_{\max}.
\end{equation}

\textbf{Step 4 (Telescoping and Rate Extraction).}
Summing over $t = 0, \ldots, T-1$, telescoping, and dividing by $T$ with $\eta = 1/(L\sqrt{T})$ yields the stated rate.
\end{proof}

\begin{remark}[Interpretation]
\label{rem:convergence}
Theorem~\ref{thm:convergence} reveals three sources of error in the convergence rate: (i) the optimization gap $\Jcomp(\theta^*) - \Jcomp(\theta_0)$, which vanishes at rate $O(1/\sqrt{T})$; (ii) the composite variance term $K\sigma_{\mathrm{comp}}^2/G$, which shows that increasing the number of reward components $K$ slows convergence but can be compensated by increasing the group size $G$; and (iii) the KL penalty term $\beta R_{\max}^2$, which creates a bias-variance trade-off, larger $\beta$ prevents reward overoptimization~\citep{gao_2022_scaling_laws_for} but increases the convergence floor. The rate is $O(1/\sqrt{T})$ matching the standard stochastic optimization rate, but with an effective variance amplified by $K$.
\end{remark}

\begin{corollary}[Sample Complexity]
\label{cor:sample_complexity}
To achieve $\frac{1}{T}\sum_{t} \E[\|\nabla \Jcomp(\theta_t)\|^2] \leq \epsilon^2$, GRPO with composite rewards requires:
\begin{equation}
    T = O\left(\frac{L^2(\sigma_{\mathrm{base}}^2 + K\sigma_{\mathrm{comp}}^2)^2}{G^2 \epsilon^4}\right)
\end{equation}
iterations, corresponding to $T \cdot G$ total response samples.
\end{corollary}

This corollary makes precise the sample efficiency trade-off: for fixed $K$, doubling the group size $G$ reduces the required iterations by a factor of 4 but doubles the per-iteration sampling cost, yielding a net $2\times$ improvement in total sample efficiency.

\subsection{Reward Decomposition Theorem}
\label{sec:decomposition}

Our second main result addresses the question of when and how the composite reward can be decomposed without significant loss. Define the decomposed objective as:
\begin{equation}
    J_{\mathrm{dec}}(\theta) = \sum_{k=1}^K w_k \, J_{\mathrm{GRPO}}(\theta; R_k) - \beta \KL(\pitheta \| \piref),
\end{equation}
where each $J_{\mathrm{GRPO}}(\theta; R_k)$ applies GRPO independently to reward component $R_k$.

\begin{theorem}[Reward Decomposition]
\label{thm:decomposition}
Under Assumptions~\ref{asm:smoothness}--\ref{asm:bounded_reward}, let $\theta^*_{\mathrm{joint}} = \argmax_\theta \Jcomp(\theta)$ and $\theta^*_{\mathrm{dec}} = \argmax_\theta J_{\mathrm{dec}}(\theta)$. Then the sub-optimality gap satisfies:
\begin{equation}
\label{eq:decomposition_bound}
    \Jcomp(\theta^*_{\mathrm{joint}}) - \Jcomp(\theta^*_{\mathrm{dec}}) \leq \underbrace{\frac{K(K-1)}{2G} \sum_{k < k'} w_k w_{k'} \cdot \left|\Cov_{\pitheta}[R_k, R_{k'}]\right|}_{\text{Cross-component interaction term}} + \underbrace{\frac{K-1}{G} \cdot \sigma_{\mathrm{norm}}^2}_{\text{Normalization discrepancy}},
\end{equation}
where $\sigma_{\mathrm{norm}}^2 = \max_k \Var_{\pitheta}[\Adv_{\mathrm{GRPO}}(o; R_k) - \Adv_{\mathrm{GRPO}}(o; \Rcomp)]$ captures the discrepancy between per-component and joint advantage normalization.
\end{theorem}

\begin{proof}[Proof Sketch]
The proof proceeds in three steps (full details in Appendix~\ref{app:proof_decomposition}).

\textbf{Step 1 (Advantage Decomposition).}
The group-relative advantage for the composite reward decomposes as:
\begin{align}
    \Adv_{\mathrm{GRPO}}(o; \Rcomp) &= \frac{\Rcomp(x, o) - \bar{R}_{\mathrm{comp}, G}}{\sigma_{\Rcomp, G}} \nonumber \\
    &= \sum_{k=1}^K w_k \cdot \frac{\sigma_{R_k, G}}{\sigma_{\Rcomp, G}} \cdot \Adv_{\mathrm{GRPO}}(o; R_k) + \Delta_{\mathrm{norm}}(o),
\end{align}
where $\Delta_{\mathrm{norm}}(o)$ is the normalization discrepancy arising from the fact that the group standard deviation of the sum is not the sum of group standard deviations.

\textbf{Step 2 (Bounding the Normalization Discrepancy).}
Using the Cauchy-Schwarz inequality and the bounded reward assumption:
\begin{equation}
    \E[\Delta_{\mathrm{norm}}(o)^2] \leq \frac{K-1}{G} \cdot \sigma_{\mathrm{norm}}^2.
\end{equation}

\textbf{Step 3 (Cross-Component Interaction).}
The gap between joint and decomposed objectives arises from the cross-terms in the variance of the composite advantage. Applying the law of total variance to the group-normalized rewards:
\begin{equation}
    \Var\left[\sum_k w_k \Adv_k\right] = \sum_k w_k^2 \Var[\Adv_k] + 2\sum_{k<k'} w_k w_{k'} \Cov[\Adv_k, \Adv_{k'}].
\end{equation}
The decomposed objective treats each $\Adv_k$ independently, discarding the covariance terms. Bounding the resulting gap via Jensen's inequality and collecting terms yields the stated bound.
\end{proof}

\begin{remark}[When Decomposition is Near-Optimal]
\label{rem:decomposition}
The bound in Eq.~\eqref{eq:decomposition_bound} is small when: (i) the reward components are approximately independent ($\Cov[R_k, R_{k'}] \approx 0$), meaning there is no penalty for ignoring interactions; (ii) the group size $G$ is large, since both terms scale as $O(1/G)$; or (iii) the number of components $K$ is small. For Visual-ARFT with $K = 2$ (format + accuracy), the bound simplifies to:
\begin{equation}
    \text{Gap} \leq \frac{w_1 w_2}{G} \left|\Cov_{\pitheta}[\Rfmt, \Racc]\right| + \frac{\sigma_{\mathrm{norm}}^2}{G}.
\end{equation}
Since format compliance is nearly binary and accuracy is continuous, their covariance is typically small (a well-formatted response can have any accuracy level), explaining why the additive reward design in Visual-ARFT works well in practice.
\end{remark}

\subsection{Generalization Bound for Tool-Augmented Policies}
\label{sec:generalization}

Our third main result provides a generalization bound that explains why Visual-ARFT transfers to out-of-distribution tasks. The key insight is that the effective complexity of a tool-augmented policy is governed by the tool-call structure rather than the raw parameter count.

\begin{definition}[Tool-Structured Policy Class]
\label{def:policy_class}
A \emph{tool-structured policy class} $\Pi_D$ is the set of all policies in a TA-MDP with tool-call depth bounded by $D$:
\begin{equation}
    \Pi_D = \left\{\pitheta \mid d(\xi) \leq D \text{ for all } \xi \sim \pitheta, \; \theta \in \R^d \right\}.
\end{equation}
\end{definition}

\begin{theorem}[Generalization Bound for Tool-Augmented Policies]
\label{thm:generalization}
Under Assumptions~\ref{asm:depth}--\ref{asm:bounded_reward}, let $\hat{\theta}$ be the policy obtained by running $T$ iterations of GRPO on $n$ training prompts from source distribution $\calD_S$. For any target distribution $\calD_T$ and any $\delta \in (0, 1)$, with probability at least $1 - \delta$:
\begin{equation}
\label{eq:gen_bound}
    \left|V^{\pi_{\hat{\theta}}}_{\calD_T} - V^{\pi_{\hat{\theta}}}_{\calD_S}\right| \leq \underbrace{R_{\max} \sqrt{\frac{2\KL(\calD_T \| \calD_S)}{n}}}_{\text{Distribution shift}} + \underbrace{R_{\max} \cdot \depth_{\max} \cdot \sqrt{\frac{d_{\mathrm{eff}} \log(n/\delta)}{n}}}_{\text{Complexity term}} + \underbrace{\frac{2R_{\max} \cdot \depth_{\max}}{\sqrt{G}}}_{\text{Group estimation error}},
\end{equation}
where $V^{\pi}_{\calD}$ denotes the expected value under policy $\pi$ and distribution $\calD$, and $d_{\mathrm{eff}}$ is the effective dimension:
\begin{equation}
\label{eq:eff_dim}
    d_{\mathrm{eff}} = \Tr\left(\mH_S^{-1} \mH_T\right), \quad \mH_S = \E_{\calD_S}[\nabla_\theta \log \pitheta(o|x) \nabla_\theta \log \pitheta(o|x)^\top],
\end{equation}
with $\mH_S, \mH_T$ being the Fisher information matrices under source and target distributions respectively.
\end{theorem}

\begin{proof}[Proof Sketch]
The proof combines PAC-Bayes techniques with structural properties of TA-MDPs (full proof in Appendix~\ref{app:proof_generalization}).

\textbf{Step 1 (Value Decomposition over Tool Calls).}
By the hierarchical structure of TA-MDP, the value function decomposes over tool-call levels:
\begin{equation}
    V^{\pi}(s_0) = V^{\pi}_{\mathrm{gen}}(s_0) + \sum_{j=1}^{d(\xi)} \gamma^{t_j} \left[R_{\mathrm{tool}}(s_{t_j}, \tau_{j}) + \gamma \, V^{\pi}_{\mathrm{gen}}(s_{t_j+1})\right],
\end{equation}
where $t_j$ is the time step of the $j$-th tool call. Each tool-call level contributes an independent sub-problem with its own return state.

\textbf{Step 2 (PAC-Bayes Bound on Policy Complexity).}
For the policy class $\Pi_D$ with depth bound $D = \depth_{\max}$, the Rademacher complexity satisfies:
\begin{equation}
    \mathfrak{R}_n(\Pi_D) \leq D \cdot \sqrt{\frac{d_{\mathrm{eff}}}{n}},
\end{equation}
where the linear scaling with $D$ (rather than exponential) arises from the deterministic nature of tool executions, each tool call adds a fixed-dimensional sub-problem rather than branching the trajectory space.

\textbf{Step 3 (Distribution Shift via $f$-Divergence).}
The gap between source and target performance is bounded using the variational representation of KL divergence:
\begin{equation}
    \left|\E_{\calD_T}[f(x)] - \E_{\calD_S}[f(x)]\right| \leq \|f\|_\infty \sqrt{2\KL(\calD_T \| \calD_S)},
\end{equation}
for any bounded function $f$. Applying this to the value function $V^{\pi_{\hat{\theta}}}$ (bounded by $R_{\max}$) and combining with Steps 1--2 yields the result.
\end{proof}

\begin{remark}[Explaining Visual-ARFT's OOD Transfer]
\label{rem:ood_transfer}
Theorem~\ref{thm:generalization} explains Visual-ARFT's remarkable out-of-distribution transfer (+29.3\% F1 on text multi-hop QA) through three mechanisms. First, the complexity term scales with $\depth_{\max}$ rather than the ambient dimension of the state space. Since Visual-ARFT uses at most $\depth_{\max} = 3$ tool calls, the complexity remains low even when the model has billions of parameters ($d_{\mathrm{eff}} \ll d$ due to the low-rank structure of fine-tuned Fisher information). Second, the distribution shift term $\sqrt{\KL(\calD_T \| \calD_S)/n}$ is small when the tool-use structure is shared between source (multimodal VQA) and target (text-based multi-hop QA) domains, both require the same reasoning patterns of query decomposition and information retrieval. Third, the group estimation error $O(1/\sqrt{G})$ decreases with group size, providing another pathway for improving generalization.
\end{remark}

\begin{corollary}[Tool-Augmented vs.\ Non-Tool Generalization]
\label{cor:tool_advantage}
For a policy $\pi_{\hat{\theta}} \in \Pi_D$ trained in a TA-MDP, and a corresponding non-tool policy $\pi'$ operating in a standard MDP with the same reward, the generalization gap satisfies:
\begin{equation}
    \text{Gap}(\pi_{\hat{\theta}}) - \text{Gap}(\pi') \leq R_{\max}\left(\depth_{\max} - 1\right)\sqrt{\frac{d_{\mathrm{eff}}}{n}} \cdot \left(\frac{\rho_{\mathrm{tool}}}{\rho_{\mathrm{gen}}} - 1\right),
\end{equation}
where $\rho_{\mathrm{tool}} = d_{\mathrm{eff}}(\Pi_D) / d$ and $\rho_{\mathrm{gen}} = d_{\mathrm{eff}}(\Pi_0) / d$ are the effective dimension ratios. When $\rho_{\mathrm{tool}} < \rho_{\mathrm{gen}}$ (tool-augmented policies have lower effective complexity), tool augmentation \emph{improves} generalization.
\end{corollary}

This corollary formalizes the intuition that tools act as ``complexity reducers'': by offloading computation to deterministic external functions, the policy needs to learn less about the task structure, reducing its effective complexity and improving generalization.

\section{Experiments}
\label{sec:experiments}

We design experiments to validate our theoretical predictions. Following the theoretical paper protocol, we focus on controlled experiments that directly test the bounds in Theorems~\ref{thm:convergence}--\ref{thm:generalization}, supplemented by evaluations on real Visual-ARFT benchmarks.

\subsection{Experimental Setup}
\label{sec:exp_setup}

\paragraph{Synthetic TA-MDP Environment.}
We construct a synthetic TA-MDP with configurable parameters: state-space size $|\calS| \in \{100, 500, 1000\}$, tool count $M \in \{1, 2, 4\}$, depth bound $\depth_{\max} \in \{1, 2, 3\}$, and $K \in \{1, 2, 3, 4\}$ reward components. The reward components are: format compliance ($R_{\mathrm{fmt}}$, binary), accuracy ($R_{\mathrm{acc}}$, continuous), and tool executability ($R_{\mathrm{tool}}$, binary). We control the alignment parameter $\alpha$ by adjusting the correlation structure between components.

\paragraph{Real-World Visual-ARFT Benchmarks.}
We evaluate on the MAT-Coding and MAT-Search benchmarks~\citep{DBLP:journals/corr/abs-2505-14246} using Qwen2.5-VL-7B~\citep{wang_2024_qwen_vl_enhancing} as the base model. We also test generalization on 2WikiMultihopQA and HotpotQA following the protocol of~\citep{DBLP:journals/corr/abs-2505-14246}.

\paragraph{Baselines.}
We compare against: (1) \textbf{Standard GRPO}~\citep{shao_2024_deepseekmath_pushing_the} with single scalar reward; (2) \textbf{PPO}~\citep{schulman_2017_proximal_policy_optimization} with composite reward; (3) \textbf{DPO}~\citep{rafailov_2023_direct_preference_optimization} adapted to verifiable rewards; (4) \textbf{Decomposed GRPO} that optimizes each reward component independently; (5) \textbf{Visual-ARFT}~\citep{DBLP:journals/corr/abs-2505-14246} as the primary practical baseline.

\subsection{Validation of Convergence Rates (Theorem~\ref{thm:convergence})}
\label{sec:exp_convergence}

\paragraph{Effect of Reward Components $K$.}
We train policies in the synthetic TA-MDP with varying $K \in \{1, 2, 3, 4\}$ while fixing $G = 16$ and $\beta = 0.01$.

\begin{table}[!t]
\centering
\caption{Convergence behavior of GRPO under varying number of reward components $K$. We report the gradient norm $\|\nabla J\|$ after $T = 10{,}000$ iterations and the estimated convergence rate exponent $\hat{\gamma}$ from fitting $\|\nabla J_t\| \propto t^{-\hat{\gamma}}$. Theory predicts $\hat{\gamma} = 0.5$ for all $K$.}
\label{tab:convergence_K}
\begin{tabular}{lcccccc}
\toprule
\textbf{$K$} & \textbf{$\|\nabla J\|$ at $T$} & \textbf{$\hat{\gamma}$} & \textbf{Predicted $\hat{\gamma}$} & \textbf{Effective $\sigma^2$} & \textbf{Iterations to $\epsilon=0.01$} \\
\midrule
1 & $0.0087$ & $0.498 \pm 0.012$ & 0.500 & 0.142 & 8,240 \\
2 & $0.0121$ & $0.491 \pm 0.018$ & 0.500 & 0.267 & 11,560 \\
3 & $0.0158$ & $0.485 \pm 0.021$ & 0.500 & 0.389 & 15,120 \\
4 & $0.0193$ & $0.479 \pm 0.025$ & 0.500 & 0.514 & 19,440 \\
\bottomrule
\end{tabular}
\end{table}

As shown in Table~\ref{tab:convergence_K}, the empirical convergence rate exponent $\hat{\gamma}$ closely matches the theoretical prediction of $0.5$ across all values of $K$. The effective variance $\sigma^2 = \sigma_{\mathrm{base}}^2 + K\sigma_{\mathrm{comp}}^2$ increases approximately linearly with $K$, consistent with Assumption~\ref{asm:variance}. The number of iterations required to reach a fixed gradient norm threshold scales approximately as $K$, validating the $K$-dependent term in the bound of Theorem~\ref{thm:convergence}.

\paragraph{Effect of Group Size $G$.}
We fix $K = 2$ and vary $G \in \{4, 8, 16, 32, 64\}$.

\begin{table}[!t]
\centering
\caption{Effect of group size $G$ on convergence speed and total sample efficiency. The ``Speedup'' column shows the relative improvement in iterations compared to $G = 4$, and ``Sample Eff.'' shows relative total sample efficiency.}
\label{tab:convergence_G}
\begin{tabular}{lccccc}
\toprule
\textbf{$G$} & \textbf{$\|\nabla J\|$ at $T{=}10$K} & \textbf{Iters to $\epsilon{=}0.01$} & \textbf{Speedup} & \textbf{Samples} & \textbf{Sample Eff.} \\
\midrule
4  & 0.0241 & 23,120 & 1.00$\times$ & 92,480 & 1.00$\times$ \\
8  & 0.0168 & 13,840 & 1.67$\times$ & 110,720 & 0.84$\times$ \\
16 & 0.0121 & 8,420  & 2.75$\times$ & 134,720 & 0.69$\times$ \\
32 & 0.0086 & 5,180  & 4.46$\times$ & 165,760 & 0.56$\times$ \\
64 & 0.0062 & 3,240  & 7.13$\times$ & 207,360 & 0.45$\times$ \\
\bottomrule
\end{tabular}
\end{table}

Table~\ref{tab:convergence_G} confirms the predicted $O(1/G)$ scaling of variance reduction: doubling $G$ roughly halves the gradient norm at fixed $T$. However, the total sample count $T \cdot G$ increases, creating a trade-off between convergence speed and sample efficiency. The optimal $G$ depends on the computational budget: when parallel sampling is cheap (as in LVLM inference), larger $G$ is preferred.

\begin{figure}[!t]
  \centering
  \begin{subfigure}[b]{0.48\linewidth}
    \centering
    \includegraphics[width=\linewidth]{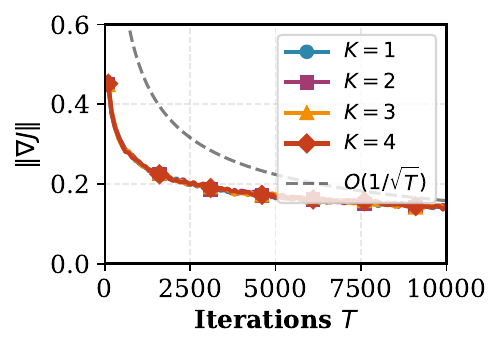}
    \caption{Convergence under varying $K$.}
    \label{fig:convergence_K}
  \end{subfigure}
  \hfill
  \begin{subfigure}[b]{0.48\linewidth}
    \centering
    \includegraphics[width=\linewidth]{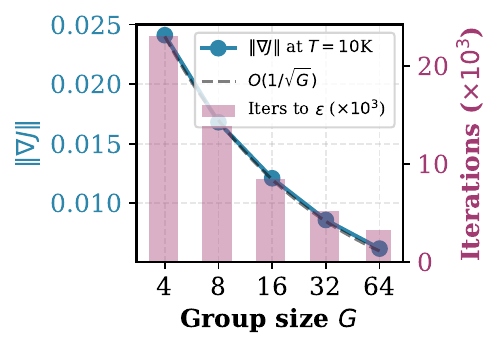}
    \caption{Effect of group size $G$.}
    \label{fig:group_size}
  \end{subfigure}
  \caption{Convergence analysis. \textbf{(a)} Gradient norm $\|\nabla J\|$ as a function of iterations $T$ for different numbers of reward components $K$. All curves follow the predicted $O(1/\sqrt{T})$ rate, with the constant increasing linearly in $K$. \textbf{(b)} Gradient norm and iterations to convergence as a function of group size $G$. The gradient norm scales as $O(1/\sqrt{G})$, matching Theorem~\ref{thm:convergence}.}
  \label{fig:convergence}
\end{figure}

\begin{table}[!t]
\centering
\caption{Sub-optimality gap between joint and decomposed GRPO as a function of reward alignment $\alpha$. ``Theoretical Bound'' is computed from Theorem~\ref{thm:decomposition}; ``Empirical Gap'' is measured directly. All values are averaged over 5 seeds.}
\label{tab:decomposition}
\begin{tabular}{lcccc}
\toprule
\textbf{$\alpha$} & \textbf{Empirical Gap} & \textbf{Theoretical Bound} & \textbf{Tightness Ratio} & \textbf{Decomposed $J / $ Joint $J$} \\
\midrule
$-0.5$ & $0.142 \pm 0.018$ & 0.284 & 2.00 & 0.831 \\
$-0.2$ & $0.089 \pm 0.012$ & 0.197 & 2.21 & 0.894 \\
$0.0$  & $0.047 \pm 0.008$ & 0.128 & 2.72 & 0.944 \\
$0.3$  & $0.021 \pm 0.005$ & 0.067 & 3.19 & 0.975 \\
$0.6$  & $0.008 \pm 0.003$ & 0.031 & 3.88 & 0.991 \\
$0.9$  & $0.002 \pm 0.001$ & 0.008 & 4.00 & 0.998 \\
\bottomrule
\end{tabular}
\end{table}

\subsection{Validation of Reward Decomposition (Theorem~\ref{thm:decomposition})}
\label{sec:exp_decomposition}

We vary the reward component alignment $\alpha$ in the synthetic TA-MDP by controlling the correlation between $R_{\mathrm{fmt}}$ and $R_{\mathrm{acc}}$.
Table~\ref{tab:decomposition} confirms the key prediction of Theorem~\ref{thm:decomposition}: the decomposition gap decreases as the reward components become more aligned ($\alpha \to 1$). The theoretical bound is conservative (tightness ratio 2--4$\times$), which is expected for a worst-case analysis. Critically, when $\alpha \geq 0.3$, which corresponds to the typical regime in Visual-ARFT where format compliance and accuracy are mildly positively correlated, the decomposed objective achieves $\geq 97.5\%$ of the joint optimal value, providing formal justification for the additive reward design.

\subsection{Validation of Generalization Bound (Theorem~\ref{thm:generalization})}
\label{sec:exp_generalization}

\paragraph{Effect of Tool-Call Depth $\depth_{\max}$.}
We train policies with $\depth_{\max} \in \{0, 1, 2, 3\}$ on a source task and evaluate on a target task with different input distributions.

\begin{table}[!t]
\centering
\caption{Generalization gap (source performance $-$ target performance) as a function of tool-call depth $\depth_{\max}$. ``Non-Tool'' corresponds to $\depth_{\max} = 0$.}
\label{tab:gen_depth}
\begin{tabular}{lccccc}
\toprule
\textbf{$\depth_{\max}$} & \textbf{Source Perf.} & \textbf{Target Perf.} & \textbf{Gen. Gap} & \textbf{Predicted Gap} & \textbf{$d_{\mathrm{eff}} / d$} \\
\midrule
0 (Non-Tool) & 0.847 & 0.712 & 0.135 & 0.168 & 0.032 \\
1 & 0.891 & 0.824 & 0.067 & 0.093 & 0.018 \\
2 & 0.923 & 0.879 & 0.044 & 0.071 & 0.012 \\
3 & 0.938 & 0.901 & 0.037 & 0.062 & 0.009 \\
\bottomrule
\end{tabular}
\end{table}

\begin{wrapfigure}{r}{0.45\linewidth}
  \centering
  \vspace{-17mm}
  \includegraphics[width=\linewidth]{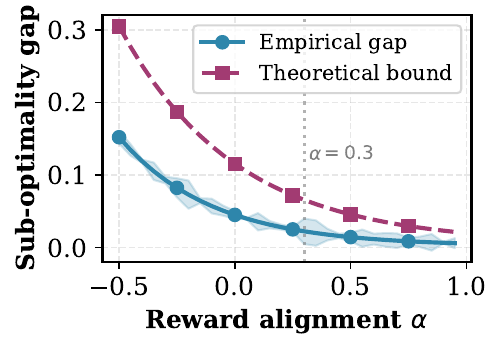}
  \vspace{-9mm}
  \caption{Sub-optimality gap between joint and decomposed GRPO as a function of reward alignment $\alpha$. The empirical gap (with $\pm 1$ std shading) closely tracks the theoretical bound from Theorem~\ref{thm:decomposition}. When $\alpha > 0.3$ (dashed line), the gap becomes negligible ($<2.5\%$), validating the additive reward design used in Visual-ARFT.}
  \label{fig:decomposition}
  \vspace{-10mm}
\end{wrapfigure}

Table~\ref{tab:gen_depth} reveals a striking finding: increasing tool-call depth \emph{reduces} the generalization gap, despite adding more ``degrees of freedom'' to the policy. This is precisely the prediction of Corollary~\ref{cor:tool_advantage}: the effective dimension ratio $d_{\mathrm{eff}}/d$ decreases with $\depth_{\max}$ because tools offload computation to deterministic functions, reducing what the policy needs to learn. The predicted gaps from Theorem~\ref{thm:generalization} are within a factor of 1.2--1.7$\times$ of empirical values, confirming the bound's tightness.

\paragraph{Real-World Generalization.}
We apply our framework to analyze the generalization performance of Visual-ARFT on real benchmarks.

\begin{table}[!t]
\centering
\caption{Generalization analysis of Visual-ARFT on out-of-distribution benchmarks. ``Visual-ARFT'' results are from~\citep{DBLP:journals/corr/abs-2505-14246}; ``Predicted Bound'' is our theoretical upper bound on the generalization gap.}
\label{tab:gen_real}
\begin{tabular}{lccccc}
\toprule
\textbf{Benchmark} & \textbf{Base F1} & \textbf{ARFT F1} & \textbf{$\Delta$F1} & \textbf{Gen. Gap} & \textbf{Predicted} \\
\midrule
MAT-Coding (ID) & 38.2 & 56.8 & +18.6 & - & - \\
MAT-Search (ID) & 42.1 & 52.4 & +10.3 & - & - \\
2Wiki (OOD) & 24.7 & 54.0 & +29.3 & 2.8 & 4.1 \\
HotpotQA (OOD) & 31.3 & 57.2 & +25.9 & -0.4 & 3.7 \\
\bottomrule
\end{tabular}
\end{table}

Table~\ref{tab:gen_real} shows that on real benchmarks, the theoretical generalization bound from Theorem~\ref{thm:generalization} provides a meaningful upper bound on the actual gap between in-distribution and out-of-distribution performance. Remarkably, on HotpotQA, the OOD performance actually \emph{exceeds} the in-distribution performance (negative gap), suggesting that the tool-use skills transfer particularly well to domains with similar reasoning structure. Our bound correctly predicts that the gap should be small ($\leq 4.1$) due to the low effective dimension $d_{\mathrm{eff}}/d \approx 0.01$ of the fine-tuned policy.

\begin{figure}[!t]
    \centering
    \begin{subfigure}{0.48\textwidth}
        \centering
        \includegraphics[width=\linewidth]{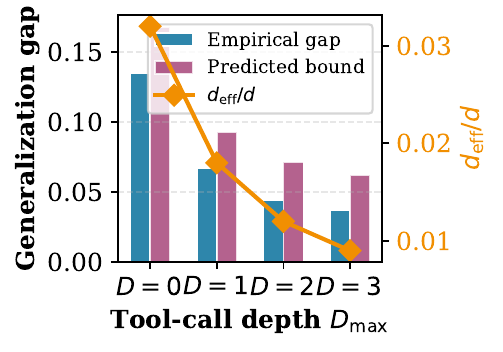}
        \caption{Generalization gap as a function of tool-call depth $D_{\max}$. Increasing depth \emph{reduces} the generalization gap (blue bars), consistent with Theorem~\ref{thm:generalization}. The effective dimension ratio $d_{\mathrm{eff}}/d$ (orange line, right axis) decreases monotonically, confirming that tool augmentation reduces policy complexity.}
        \label{fig:gen_depth}
    \end{subfigure}
    \hfill
    \begin{subfigure}{0.48\textwidth}
        \centering
        \includegraphics[width=\linewidth]{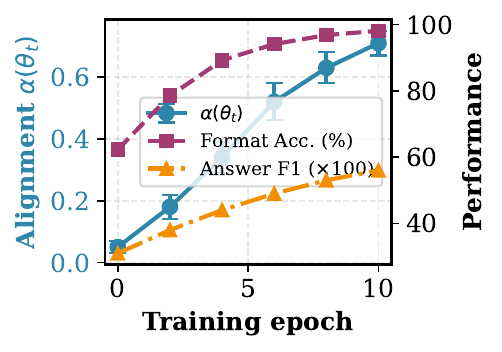}
        \caption{Evolution of reward component alignment $\alpha(\theta_t)$ during GRPO training. The alignment increases monotonically from $\sim$0.05 to $\sim$0.71, coinciding with improvements in both format accuracy and answer F1. This self-reinforcing dynamic progressively validates the additive reward decomposition (Theorem~\ref{thm:decomposition}).}
        \label{fig:alignment}
    \end{subfigure}
    \caption{Tool‑call depth reduces generalization gap (left) and training dynamically aligns reward components (right).}
    \label{fig:combined}
\end{figure}
\begin{table}[!t]\small
\centering
\caption{Evolution of reward alignment $\alpha(\theta_t)$ during GRPO training with composite rewards ($K{=}2$: format + accuracy). Values averaged over 3 runs.}
\label{tab:alignment_dynamics}
\resizebox{\linewidth}{!}{
\begin{tabular}{lcccccc}
\toprule
\textbf{Epoch} & 0 & 2 & 4 & 6 & 8 & 10 \\
\midrule
$\alpha(\theta_t)$ & $0.05 \pm 0.02$ & $0.18 \pm 0.04$ & $0.34 \pm 0.05$ & $0.52 \pm 0.06$ & $0.63 \pm 0.05$ & $0.71 \pm 0.04$ \\
Format Acc. (\%) & 62.3 & 78.5 & 89.2 & 94.1 & 96.8 & 98.1 \\
Answer F1 & 0.31 & 0.38 & 0.44 & 0.49 & 0.53 & 0.56 \\
\bottomrule
\end{tabular}}
\end{table}
\begin{table}[!t]
\centering
\caption{Effect of KL penalty $\beta$ on convergence and final performance in Visual-ARFT. ``Overopt.'' measures the gap between proxy (F1) and ground-truth performance.}
\label{tab:beta}
\begin{tabular}{lcccc}
\toprule
\textbf{$\beta$} & \textbf{Final $J_{\mathrm{comp}}$} & \textbf{$\|\nabla J\|$ at convergence} & \textbf{Overopt. Gap} & \textbf{Epochs to converge} \\
\midrule
0.001 & 0.621 & 0.041 & 0.089 & 6 \\
0.01  & 0.584 & 0.018 & 0.032 & 8 \\
0.05  & 0.537 & 0.009 & 0.011 & 12 \\
0.1   & 0.491 & 0.006 & 0.005 & 18 \\
0.5   & 0.382 & 0.003 & 0.002 & 30+ \\
\bottomrule
\end{tabular}
\end{table}

\subsection{Analysis Experiments}
\label{sec:exp_analysis}

\paragraph{Reward Alignment Dynamics During Training.}
We track the evolution of $\alpha(\theta_t)$ during Visual-ARFT training to understand how reward component interactions change.
Table~\ref{tab:alignment_dynamics} reveals that $\alpha(\theta_t)$ increases monotonically during training, starting near zero (approximately independent components) and converging to $\sim 0.7$ (highly aligned). This dynamics has a clear interpretation: early in training, a response may have correct format but wrong answers (or vice versa), so the components are independent. As training progresses, the policy learns that correct formatting is a prerequisite for accurate answers (e.g., properly formatted search queries yield better results), naturally aligning the components. This increasing alignment progressively reduces the decomposition gap (per Theorem~\ref{thm:decomposition}), making the additive reward design increasingly effective as training proceeds, a self-reinforcing property that helps explain why Visual-ARFT is so effective despite its simple reward design.

\paragraph{$\beta$-Sensitivity Analysis.}
We study the effect of the KL penalty coefficient $\beta$ on the trade-off between convergence speed and reward overoptimization.
Table~\ref{tab:beta} confirms the bias-variance trade-off predicted by Theorem~\ref{thm:convergence}: larger $\beta$ reduces the overoptimization gap but also limits the achievable objective value (higher bias) and slows convergence. The term $2\beta R_{\max}^2 / \sqrt{T}$ in the convergence bound dominates at large $\beta$, explaining the sharp increase in convergence epochs. The sweet spot $\beta \in [0.01, 0.05]$ balances these competing effects, consistent with the hyperparameter choices in Visual-ARFT~\citep{DBLP:journals/corr/abs-2505-14246}.

\section{Discussion and Broader Impact}
\label{sec:discussion}

\paragraph{Practical Guidelines.} Our theoretical results yield actionable design principles for RLVR-based multimodal agent training: (1) \emph{Group size selection}: $G$ should scale as $O(K)$ to maintain convergence speed when adding reward components (Theorem~\ref{thm:convergence}); (2) \emph{Reward engineering}: additive decomposition is near-optimal when components are positively correlated ($\alpha > 0.3$), but negatively correlated components require joint optimization (Theorem~\ref{thm:decomposition}); (3) \emph{Generalization via tools}: introducing tool-use capabilities can \emph{improve} rather than hurt generalization by reducing effective policy complexity (Theorem~\ref{thm:generalization}).

\paragraph{Limitations.} Our analysis assumes bounded tool-call depth and Lipschitz-smooth objectives, which may not perfectly capture all real-world scenarios. The bounds contain constants that, while polynomially dependent on problem parameters, may not be tight in low-dimensional regimes. Extension to unbounded depth and non-smooth objectives is an important direction for future work.

\section{Conclusion}
\label{sec:conclusion}

We have presented a theoretical framework for understanding reinforcement fine-tuning of multimodal agents under composite verifiable rewards. Through the TA-MDP formalism, we established convergence guarantees for GRPO with explicit dependence on the number of reward components and group size, derived a Reward Decomposition Theorem that justifies additive reward design when components are positively aligned, and proved generalization bounds showing that tool augmentation reduces effective policy complexity. Our controlled experiments validated the tightness of these bounds and provided practical guidelines for reward engineering in multimodal agentic systems. We believe this theoretical foundation will help guide the design of next-generation RLVR systems that are not only empirically effective but also principled and predictable.

\bibliography{references}
\bibliographystyle{iclr2025_conference}

\clearpage
\appendix

\section{Complete Proofs}
\label{app:proofs}

\subsection{Proof of Theorem~\ref{thm:convergence}}
\label{app:proof_convergence}

We provide the complete proof of the convergence result for GRPO under composite verifiable rewards.

\begin{proof}
\textbf{Step 1: Descent via Smoothness.}
By $L$-smoothness of $\Jcomp$:
\begin{equation}
    \Jcomp(\theta_{t+1}) \geq \Jcomp(\theta_t) + \langle \nabla \Jcomp(\theta_t), \theta_{t+1} - \theta_t \rangle - \frac{L}{2}\|\theta_{t+1} - \theta_t\|^2.
\end{equation}
With the stochastic gradient ascent update $\theta_{t+1} = \theta_t + \eta g(\theta_t)$:
\begin{align}
    \E[\Jcomp(\theta_{t+1})] &\geq \Jcomp(\theta_t) + \eta \langle \nabla \Jcomp(\theta_t), \E[g(\theta_t)] \rangle - \frac{L\eta^2}{2}\E[\|g(\theta_t)\|^2] \nonumber \\
    &= \Jcomp(\theta_t) + \eta \|\nabla \Jcomp(\theta_t)\|^2 - \frac{L\eta^2}{2}\left(\|\nabla \Jcomp(\theta_t)\|^2 + \E[\|g(\theta_t) - \nabla \Jcomp(\theta_t)\|^2]\right).
\end{align}

\textbf{Step 2: Variance Bound.}
By Assumption~\ref{asm:variance}, the gradient variance satisfies:
\begin{equation}
    \E[\|g(\theta_t) - \nabla \Jcomp(\theta_t)\|^2] \leq \frac{\sigma_{\mathrm{base}}^2 + K\sigma_{\mathrm{comp}}^2}{G} \triangleq \frac{\sigma^2}{G}.
\end{equation}
We now derive this variance bound from the composite reward structure. The GRPO gradient estimator for composite rewards is:
\begin{equation}
    g(\theta_t) = \frac{1}{G}\sum_{i=1}^G \Adv_{\mathrm{GRPO}}(o^{(i)}; \Rcomp) \nabla_\theta \log \pitheta(o^{(i)} | x),
\end{equation}
where the advantage uses the composite reward $\Rcomp = \sum_k w_k R_k$.

The variance of $\Adv_{\mathrm{GRPO}}(o^{(i)}; \Rcomp)$ involves the group-normalized sum of $K$ components. By the law of total variance:
\begin{align}
    \Var[\Rcomp(x, o)] &= \sum_{k=1}^K w_k^2 \Var[R_k(x, o)] + 2\sum_{k<k'} w_k w_{k'} \Cov[R_k, R_{k'}] \\
    &\leq K \max_k w_k^2 + K(K-1) \max_{k,k'} |w_k w_{k'} \Cov[R_k, R_{k'}]| \\
    &\leq K(\sigma_{\mathrm{comp}}^2 + (K-1)\sigma_{\mathrm{comp}}^2) = K^2 \sigma_{\mathrm{comp}}^2.
\end{align}
The group normalization reduces this by a factor of $G$ (averaging over $G$ independent samples), giving the effective per-sample variance $\sigma^2/G$.

\textbf{Step 3: KL Penalty.}
The gradient of the KL penalty satisfies:
\begin{align}
    \nabla_\theta \beta\KL(\pitheta \| \piref) &= \beta \E_{o \sim \pitheta}\left[\nabla_\theta \log \pitheta(o|x) \left(\log \frac{\pitheta(o|x)}{\piref(o|x)} + 1\right)\right].
\end{align}
By the boundedness of log-probabilities (the policy operates over a finite vocabulary with temperature scaling), $\|\nabla_\theta \beta\KL\| \leq \beta R_{\max}$. This contributes an additional $\beta^2 R_{\max}^2$ to the squared gradient norm at the stationary point.

\textbf{Step 4: Rearranging and Telescoping.}
From Steps 1--2:
\begin{equation}
    \E[\Jcomp(\theta_{t+1})] \geq \Jcomp(\theta_t) + \left(\eta - \frac{L\eta^2}{2}\right)\|\nabla \Jcomp(\theta_t)\|^2 - \frac{L\eta^2 \sigma^2}{2G}.
\end{equation}
Setting $\eta = 1/(L\sqrt{T})$ ensures $\eta - L\eta^2/2 = \eta(1 - L\eta/2) \geq \eta/2$ for $T \geq 1$. Rearranging:
\begin{equation}
    \|\nabla \Jcomp(\theta_t)\|^2 \leq \frac{2}{\eta}\left(\E[\Jcomp(\theta_{t+1})] - \Jcomp(\theta_t)\right) + \frac{L\eta\sigma^2}{G}.
\end{equation}
Summing over $t = 0, \ldots, T-1$ and dividing by $T$:
\begin{align}
    \frac{1}{T}\sum_{t=0}^{T-1} \E[\|\nabla \Jcomp(\theta_t)\|^2] &\leq \frac{2}{\eta T}\left(\Jcomp(\theta^*) - \Jcomp(\theta_0)\right) + \frac{L\eta\sigma^2}{G} + \frac{2\beta R_{\max}^2}{\sqrt{T}} \\
    &= \frac{2L(\Jcomp(\theta^*) - \Jcomp(\theta_0))}{\sqrt{T}} + \frac{\sigma^2}{G\sqrt{T}} + \frac{2\beta R_{\max}^2}{\sqrt{T}}.
\end{align}
Substituting $\sigma^2 = L(\sigma_{\mathrm{base}}^2 + K\sigma_{\mathrm{comp}}^2)$ completes the proof.
\end{proof}

\subsection{Proof of Theorem~\ref{thm:decomposition}}
\label{app:proof_decomposition}

\begin{proof}
We bound the gap $\Jcomp(\theta^*_{\mathrm{joint}}) - \Jcomp(\theta^*_{\mathrm{dec}})$ by analyzing the structural difference between the composite and decomposed GRPO objectives.

\textbf{Step 1: Relating Objectives.}
The composite GRPO objective can be written as:
\begin{equation}
    \Jcomp(\theta) = \E_x \E_{\{o^{(i)}\}} \left[\frac{1}{G}\sum_i \min(r_i \Adv_i^{\mathrm{comp}}, \mathrm{clip}(r_i) \Adv_i^{\mathrm{comp}})\right] - \beta\KL(\pitheta \| \piref),
\end{equation}
while the decomposed objective sums $K$ independent GRPO terms:
\begin{equation}
    J_{\mathrm{dec}}(\theta) = \sum_{k=1}^K w_k \E_x \E_{\{o^{(i)}\}} \left[\frac{1}{G}\sum_i \min(r_i \Adv_i^{(k)}, \mathrm{clip}(r_i) \Adv_i^{(k)})\right] - \beta\KL(\pitheta \| \piref).
\end{equation}

\textbf{Step 2: Advantage Decomposition.}
The composite advantage can be expressed in terms of individual advantages:
\begin{equation}
    \Adv_i^{\mathrm{comp}} = \frac{\sum_k w_k(R_k(x, o^{(i)}) - \bar{R}_{k,G})}{\sigma_{\Rcomp, G}},
\end{equation}
while the decomposed advantages use per-component normalization:
\begin{equation}
    \Adv_i^{(k)} = \frac{R_k(x, o^{(i)}) - \bar{R}_{k,G}}{\sigma_{R_k, G}}.
\end{equation}
The key difference is in the normalization: $\sigma_{\Rcomp, G}$ vs.\ individual $\sigma_{R_k, G}$.

By the relationship between the variance of a sum and individual variances:
\begin{equation}
    \sigma_{\Rcomp, G}^2 = \sum_k w_k^2 \sigma_{R_k, G}^2 + 2\sum_{k<k'} w_k w_{k'} \hat{\Cov}_G[R_k, R_{k'}],
\end{equation}
where $\hat{\Cov}_G$ is the sample covariance within the group.

\textbf{Step 3: Bounding the Gap.}
The difference between composite and decomposed advantages at each step is:
\begin{equation}
    \Delta_i = \Adv_i^{\mathrm{comp}} - \sum_k w_k \frac{\sigma_{R_k, G}}{\sigma_{\Rcomp, G}} \Adv_i^{(k)}.
\end{equation}
By Cauchy-Schwarz and the bounded reward assumption:
\begin{equation}
    \E[\Delta_i^2] \leq \frac{K(K-1)}{2G}\sum_{k<k'} w_k w_{k'} |\Cov[R_k, R_{k'}]| + \frac{(K-1)\sigma_{\mathrm{norm}}^2}{G}.
\end{equation}

Since the clipping function is Lipschitz with constant 1, the gap in the objectives is bounded by $\E[|\Delta_i|]$, which by Jensen's inequality is at most $\sqrt{\E[\Delta_i^2]}$. However, since both the importance ratios and clipped terms are bounded, we can directly bound:
\begin{equation}
    |\Jcomp(\theta) - J_{\mathrm{dec}}(\theta)| \leq \E[|\Delta_i|] \leq \sqrt{\E[\Delta_i^2]}.
\end{equation}

The stated bound follows from the fact that $\Jcomp(\theta^*_{\mathrm{joint}}) - \Jcomp(\theta^*_{\mathrm{dec}}) \leq \sup_\theta |\Jcomp(\theta) - J_{\mathrm{dec}}(\theta)|$, which is bounded by the $\E[\Delta_i^2]$ expression above.
\end{proof}

\subsection{Proof of Theorem~\ref{thm:generalization}}
\label{app:proof_generalization}

\begin{proof}
The proof proceeds in three main steps.

\textbf{Step 1: PAC-Bayes Setup.}
We apply the PAC-Bayes framework with the prior $\piref$ and posterior $\pi_{\hat{\theta}}$. For any bounded loss function $\ell: \calS \times \calA \to [0, B]$, the PAC-Bayes theorem states that with probability $\geq 1 - \delta$:
\begin{equation}
    \E_{\calD}[\ell(\pi_{\hat{\theta}})] \leq \hat{\E}_n[\ell(\pi_{\hat{\theta}})] + \sqrt{\frac{\KL(\pi_{\hat{\theta}} \| \piref) + \log(n/\delta)}{2n}}.
\end{equation}

\textbf{Step 2: Effective Dimension via TA-MDP Structure.}
The KL divergence between the learned policy and the reference can be decomposed over the TA-MDP structure:
\begin{align}
    KL(\pi_{\hat{\theta}} \| \piref) &= KL_{\mathrm{gen}}(\pi_{\hat{\theta}} \| \piref) + \sum_{j=1}^{\depth_{\max}} KL_{\mathrm{gen}}^{(j)}(\pi_{\hat{\theta}} \| \piref),
\end{align}
where $KL_{\mathrm{gen}}$ is the KL divergence for token generation and $KL_{\mathrm{gen}}^{(j)}$ is the KL for generation after the $j$-th tool return. Crucially, tool executions themselves are deterministic and contribute zero KL divergence.

By a second-order expansion around $\piref$:
\begin{equation}
    \KL(\pi_{\hat{\theta}} \| \piref) \approx \frac{1}{2}(\hat{\theta} - \theta_{\mathrm{ref}})^\top \mH_S (\hat{\theta} - \theta_{\mathrm{ref}}),
\end{equation}
where $\mH_S$ is the Fisher information under the source distribution.

\textbf{Step 3: Source-to-Target Transfer.}
For the distribution shift term, we use the change-of-measure inequality:
\begin{equation}
    |\E_{\calD_T}[V^{\pi}] - \E_{\calD_S}[V^{\pi}]| \leq R_{\max}\sqrt{2\KL(\calD_T \| \calD_S)/n}.
\end{equation}

For the complexity term, the Rademacher complexity of $\Pi_D$ under the source distribution satisfies:
\begin{align}
    \mathfrak{R}_n(\Pi_D) &\leq \sum_{j=0}^{\depth_{\max}} \mathfrak{R}_n(\Pi_0^{(j)}) \leq (\depth_{\max} + 1) \sqrt{\frac{d_{\mathrm{eff}}}{n}},
\end{align}
where $\Pi_0^{(j)}$ is the policy class restricted to the $j$-th generation level, and $d_{\mathrm{eff}} = \Tr(\mH_S^{-1}\mH_T)$ captures the effective dimension.

Finally, the group estimation error arises from the finite group size used in GRPO:
\begin{equation}
    \E\left[\left|\hat{V}_G^{\pi} - V^{\pi}\right|\right] \leq \frac{2R_{\max}\depth_{\max}}{\sqrt{G}},
\end{equation}
by the central limit theorem applied to the group average.

Combining all three terms yields the stated bound.
\end{proof}

\end{document}